\documentclass{article}

\usepackage[final]{nips_2017}

\usepackage[utf8]{inputenc} 
\usepackage[T1]{fontenc}    
\usepackage{hyperref}       
\usepackage{url}            
\usepackage{booktabs}       
\usepackage{amsfonts}       
\usepackage{nicefrac}       
\usepackage{microtype}      

\usepackage{amsfonts}
\usepackage{amssymb}
\usepackage{amsmath}
\usepackage{amsthm}
\usepackage{graphicx}
\usepackage{subfigure}
\usepackage{algorithm}
\usepackage{algorithmic}

\usepackage[font=footnotesize]{caption}

\usepackage{color}

\usepackage{tikz}
\theoremstyle{plain}
\newtheorem{theorem}{Theorem}
\newtheorem{proposition}[theorem]{Proposition}

\title{Learning to Pivot with Adversarial Networks}

\author{
  Gilles Louppe\\
  New York University\\
  \texttt{g.louppe@nyu.edu}\\
  \And
  Michael Kagan\\
  SLAC National Accelerator Laboratory\\
  \texttt{makagan@slac.stanford.edu}\\
  \And
  Kyle Cranmer\\
  New York University\\
  \texttt{kyle.cranmer@nyu.edu}
}

\begin{document}

\maketitle

\begin{abstract}

Several techniques for domain adaptation have been proposed to account for
differences in the distribution of the data used for training and testing. The
majority of this work focuses on a binary domain label. Similar problems occur
in a scientific context where there may be a continuous family of plausible data
generation processes associated to the presence of systematic uncertainties.
Robust inference is possible if it is based on a pivot -- a quantity whose
distribution does not depend on the unknown values of the nuisance parameters
that parametrize this family of data generation processes. In this work,  we
introduce and derive theoretical results for a training procedure based on
adversarial networks for enforcing the pivotal property (or, equivalently,
fairness with respect to continuous attributes) on a predictive model. The
method includes a hyperparameter to control the trade-off between accuracy and
robustness. We demonstrate the effectiveness of this approach with a toy example
and examples from particle physics.

\end{abstract}



\section{Introduction}

Machine learning techniques have been used to enhance a number of scientific
disciplines, and they have the potential to transform even more of the
scientific process. One of the challenges of applying machine learning
to scientific problems is the need to incorporate systematic
uncertainties, which affect both the
robustness of inference and the metrics used to evaluate a
particular analysis strategy.

In this work, we focus on supervised learning techniques where
systematic uncertainties can be associated to a data generation
process that is not uniquely specified. In other words, the lack of systematic
uncertainties corresponds to the (rare) case that the process that generates
training data is unique, fully specified, and an accurate representative of the
real world data. By contrast, a common situation when systematic uncertainty is
present is when the training data are not representative of the real data.
Several techniques for domain
adaptation have been developed to create models that are more robust
to this binary type of uncertainty.
A more generic situation is that there are several
plausible data generation processes, specified as a family
parametrized by continuous nuisance parameters, as is typically found in scientific domains.
In this broader context, statisticians have for long been working on robust inference techniques
based on the concept of a pivot -- a quantity whose distribution is invariant
with the nuisance parameters (see e.g., \citep{degroot1986probability}).

Assuming a probability model $p(X,Y,Z)$, where $X$ are the data, $Y$ are the
target labels, and $Z$ are the nuisance parameters, we consider the problem of
learning a predictive model $f(X)$ for $Y$ conditional on the observed values of $X$
that is robust to uncertainty in the unknown value of $Z$. We
introduce a flexible learning procedure based on adversarial networks~\citep{goodfellow2014generative}
for enforcing that  $f(X)$ is a pivot with respect to $Z$.
We derive theoretical results proving that the procedure
converges towards a model that is both optimal and statistically independent of the
nuisance parameters (if that model exists) or for which one can tune a
trade-off between accuracy and robustness (e.g., as driven by a higher level objective).
In particular, and to the best of our knowledge, our contribution is the first
solution for imposing pivotal constraints on a predictive model,
working regardless of the type of
the nuisance parameter (discrete or continuous) or of its prior.
Finally, we demonstrate the
effectiveness of the approach with a toy example and examples from particle physics.

\tikzstyle{every node}=[font=\scriptsize]
\begin{figure*}
   \usetikzlibrary{arrows}
   \def\layersep{1cm}

   \begin{tikzpicture}[shorten >= 1pt, ->, node distance=\layersep,scale=.65]
   \tikzstyle{neuron} = [circle, minimum size=0.25cm, draw=black!20, line width=0.3mm, fill=white]

   \node at (2,0) {Classifier $f$};
   \draw (-1,-0.5) rectangle (4,-5.5);

   \path[->, shorten >= 0pt] (-2,-3) edge (-1,-3);
   \node[left] at (-2,-3) {$X$};

   \path[-o, shorten >= 0pt] (1.5,-6.5) edge (1.5,-5.5);
   \node[below] at (1.5,-6.5) {$\theta_f$};

   \path[->, shorten >= 0pt] (3.5,-3) edge (6.5,-3);
   \node[above] at (5.25,-3) {$f(X;\theta_f)$};

   \path[dashed,-] (5.25,-3) edge (5.25,-6.5);
   \node[below] at (5.25,-6.5) {${\cal L}_f(\theta_f)$};

   \foreach \name / \y in {1,...,3}
       \node[neuron] (f-I-\name) at (-0.5,-1-\y) {};

   \foreach \name / \y in {1,...,5}
       \node[neuron] (f-H1-\name) at (-0.5cm+\layersep,-\y cm) {};
   \foreach \name / \y in {1,...,5}
       \node[neuron] (f-H2-\name) at (-0.5cm+3*\layersep,-\y cm) {};

   \node[neuron] (f-O) at (-0.5cm+4*\layersep,-3cm) {};

   \foreach \source in {1,...,3}
       \foreach \dest in {1,...,5}
           \path[black] (f-I-\source) edge (f-H1-\dest);

   \foreach \source in {1,...,5}
       \path[black] (f-H2-\source) edge (f-O);

   \node[black] at (1.5,-3) {...};

   \node at (11.75,0) {Adversary $r$};
   \draw (6.5,-0.5) rectangle (11.5,-5.5);

   \node[above] at (13.25,-2) {$\gamma_1(f(X;\theta_f);\theta_r)$};
   \path[-o, shorten >= 0pt] (11,-2.0) edge (15,-2.0);
   \node[above] at (13.25,-3) {$\gamma_2(f(X;\theta_f);\theta_r)$};
   \path[-o, shorten >= 0pt] (11,-3) edge (15,-3);
   \node[above] at (13.25,-4) {$\dots$};
   \path[-o, shorten >= 0pt] (11,-4) edge (15,-4);

   \path[-o, shorten >= 0pt] (9,-6.5) edge (9,-5.5);
   \node[below] at (9,-6.5) {$\theta_r$};

   \foreach \name / \y in {1,...,1}
       \node[neuron] (r-I-\name) at (7,-2-\y) {};

   \foreach \name / \y in {1,...,5}
       \node[neuron] (r-H1-\name) at (7cm+\layersep,-\y cm) {};
   \foreach \name / \y in {1,...,5}
       \node[neuron] (r-H2-\name) at (7cm+3*\layersep,-\y cm) {};

   \node[neuron] (r-O1) at (7cm+4*\layersep,-2cm) {};
   \node[neuron] (r-O2) at (7cm+4*\layersep,-3cm) {};
   \node[neuron] (r-O3) at (7cm+4*\layersep,-4cm) {};

   \foreach \source in {1,...,1}
       \foreach \dest in {1,...,5}
           \path[black] (r-I-\source) edge (r-H1-\dest);

   \foreach \source in {1,...,5}
       \path[black] (r-H2-\source) edge (r-O1);
   \foreach \source in {1,...,5}
       \path[black] (r-H2-\source) edge (r-O2);
   \foreach \source in {1,...,5}
       \path[black] (r-H2-\source) edge (r-O3);

   \node[black] at (9,-3) {...};

   \draw (15,-1.5) rectangle (18,-4.5);
   \path[->, shorten >= 0pt] (16.5,-0.5) edge (16.5,-1.5);
   \node[above] at (16.5,-0.5) {$Z$};
   \path[->, shorten >= 0pt] (16.5,-4.5) edge (16.5,-5.5);
   \node[below] at (16.5,-5.5) {$p_{\theta_r}(Z|f(X;\theta_f))$};
   \node at (16.5,-3) {${\cal P}(\gamma_1, \gamma_2, \dots)$};

   \draw[dashed,-] (16.5,-5) -| (12.75,-6.5);
   \node[below] at (12.75,-6.5) {${\cal L}_r(\theta_f, \theta_r)$};

   \end{tikzpicture}

    \caption{Architecture for the adversarial training of a binary classifier
    $f$ against a nuisance parameters $Z$. The adversary $r$ models the
    distribution $p(z|f(X;\theta_f)=s)$ of the nuisance parameters as observed only through the output $f(X;\theta_f)$ of the classifier. By
    maximizing the antagonistic objective ${\cal L}_r(\theta_f, \theta_r)$, the classifier
    $f$ forces $p(z|f(X;\theta_f)=s)$ towards the prior $p(z)$, which happens
    when $f(X;\theta_f)$ is  independent of the nuisance parameter $Z$ and therefore pivotal.}

    \label{fig:architecture}
\end{figure*}
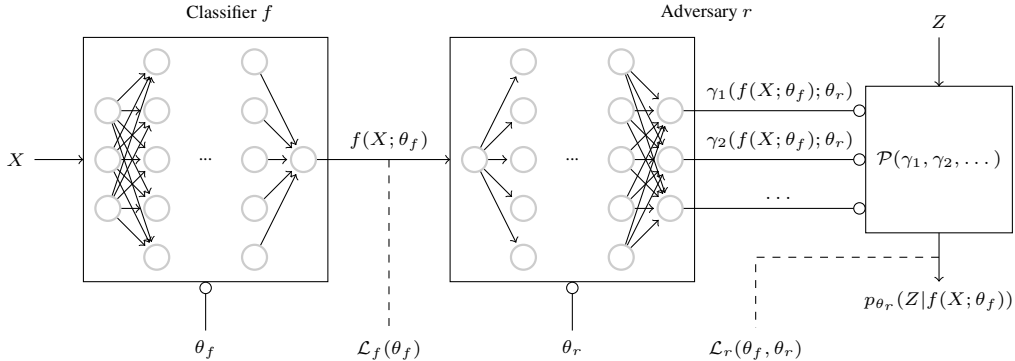



\section{Problem statement}
\label{sec:problem}

We begin with a family of data generation processes $p(X,Y,Z)$, where $X \in
\mathcal{X}$ are the data, $Y\in \mathcal{Y}$ are the target labels, and $Z\in
\mathcal{Z}$ are the nuisance parameters that can be continuous or categorical.
Let us assume that prior to incorporating the effect of uncertainty
in $Z$, our goal is to learn a regression function $f : \mathcal{X} \to
{\cal S}$ with parameters $\theta_f$ (e.g., a neural network-based probabilistic
classifier) that minimizes a loss ${\cal L}_f(\theta_f)$ (e.g., the
cross-entropy). In classification, values $\smash{s \in {\cal S} = \mathbb{R}^{|{\cal Y}|}}$
correspond to the classifier scores used for mapping hard predictions $y \in {\cal Y}$, while ${\cal S} = {\cal Y}$ for regression.

We augment our initial objective so that inference based on $f(X ; \theta_f)$ will be
robust to the value $z \in {\cal Z}$ of the nuisance parameter $Z$  -- which remains unknown at
test time. A formal way of enforcing robustness is to require that the distribution of
$f(X ; \theta_f)$ conditional on $Z$ (and possibly $Y$) be invariant with
 the nuisance parameter $Z$. Thus, we wish to find a function $f$ such that
\begin{equation}\label{eqn:criterion}
    p(f(X ; \theta_f) = s | z ) = p(f(X ; \theta_f) = s | z^\prime )
\end{equation}
for all $z,z^\prime \in  {\cal Z}$ and all values $s \in {\cal S}$ of $f(X ; \theta_f)$.
In words, we are looking for a predictive function $f$
which is a pivotal quantity with respect to the
nuisance parameters. This implies that  $f(X; \theta_f)$ and $Z$ are independent random variables.

As stated in Eqn.~\ref{eqn:criterion}, the pivotal quantity criterion is
imposed with respect to $p(X|Z)$ where $Y$ is marginalized out. In some situations
however (see e.g., Sec.~\ref{sec:hep}), class conditional independence of $f(X;
\theta_f)$ on the nuisance $Z$ is preferred, which can then be stated as
requiring
\begin{equation}\label{eqn:criterion-class}
    p(f(X ; \theta_f) = s | z, y ) = p(f(X ; \theta_f) = s | z^\prime, y )
\end{equation}
for one or several specified values $y \in {\cal Y}$.


\section{Method}
\label{sec:method}

\begin{figure*}
    \begin{minipage}{\linewidth}
    \begin{algorithm}[H]
    \caption{Adversarial training of a classifier $f$ against an adversary $r$.}

    {\footnotesize
    \begin{flushleft}
        {\it Inputs:} training data $\{ x_i, y_i, z_i \}_{i=1}^N$;
        {\it Outputs:} $\smash{\hat\theta_f}, \smash{\hat\theta_r}$.
    \end{flushleft}

    \label{alg:adversarial-training}
    \begin{algorithmic}[1]
        \FOR{$t=1$ to $T$}
            \FOR{$k=1$ to $K$} 
                \STATE{Sample minibatch $\{x_m, z_m, s_m = f(x_m;\theta_f) \}_{m=1}^M$ of size $M$;}
                \STATE{With $\theta_f$ fixed, update $r$ by ascending its stochastic gradient $\nabla_{\theta_r} E(\theta_f, \theta_r) :=$
                $$\nabla_{\theta_r} \sum_{m=1}^M \log p_{\theta_r}(z_m|s_m)  ;$$}
            \ENDFOR
            \STATE{Sample minibatch $\{x_m, y_m, z_m, s_m = f(x_m;\theta_f)  \}_{m=1}^M$ of size $M$;} 
            \STATE{With $\theta_r$ fixed, update $f$ by descending its stochastic gradient $\nabla_{\theta_f} E(\theta_f, \theta_r) :=$
            $$\nabla_{\theta_f}  \sum_{m=1}^M \left[ -\log p_{\theta_f}(y_m|x_m)  +\log p_{\theta_r}(z_m|s_m)  \right],$$
            \indent where $p_{\theta_f}(y_m|x_m)$ denotes $\mathbf{1}(y_m=0)(1-s_m) + \mathbf{1}(y_m=1)s_m$;}
        \ENDFOR
    \end{algorithmic}
    }
    \end{algorithm}
    \end{minipage}
\end{figure*}

Joint training of adversarial networks was first proposed by \citep{goodfellow2014generative} as a
way to build a generative model capable of producing samples from random noise
$z$. More specifically, the authors pit a generative model $g:
\mathbb{R}^n \to \mathbb{R}^p$ against an adversarial classifier $d :
\mathbb{R}^p \to [0, 1]$ whose antagonistic objective is to recognize
real data $X$ from generated data $g(Z)$. Both models $g$ and $d$ are trained
simultaneously, in such a way that $g$ learns to produce samples that are
difficult to identify by $d$, while $d$ incrementally adapts to changes in $g$.
At the equilibrium, $g$ models a distribution whose samples can be identified by
$d$ only by chance. That is, assuming enough capacity in $d$ and  $g$, the
distribution of $g(Z)$ eventually converges towards the real distribution
of $X$.

In this work, we repurpose adversarial networks as a means to constrain the
predictive model $f$ in order to satisfy Eqn.~\ref{eqn:criterion}. As
illustrated in Fig.~\ref{fig:architecture}, we pit $f$ against an adversarial
model $r := p_{\theta_r}(z | f(X;\theta_f)=s)$ with parameters $\theta_r$ and
associated loss ${\cal L}_r(\theta_f, \theta_r)$. This model takes  as input
realizations $s$ of $f(X; \theta_f)$ and produces as output a function
modeling the posterior probability density $p_{\theta_r}(z | f(X;\theta_f)=s)$.
Intuitively, if $p(f(X; \theta_f)=s|z)$ varies with $z$,
then the corresponding correlation can be captured by $r$. By contrast, if
$p(f(X; \theta_f)=s|z)$ is invariant with $z$, as we require, then $r$ should
perform poorly and be close to random guessing. Training $f$ such that it
additionally minimizes the performance of $r$ therefore acts as a regularization
towards Eqn.~\ref{eqn:criterion}.

If $Z$ takes discrete values, then $p_{\theta_r}$ can be represented as a
probabilistic classifier $\mathbb{R} \to \mathbb{R}^{|{\cal Z|}}$ whose
$j^\textrm{th}$ output (for $j=1, \dots, |{\cal Z}|$) is the estimated probability mass
$p_{\theta_r}(z_j|f(X;\theta_f)=s)$. Similarly, if $Z$ takes continuous values,
then we can model the posterior probability density $p(z | f(X;\theta_f)=s)$
with a sufficiently flexible  parametric family of distributions $\mathcal{P}(\gamma_1, \gamma_2, \dots)$,
where the parameters $\gamma_j$ depend on $f(X, \theta_f)$ and $\theta_r$.
The adversary $r$ may
take any form, i.e. it does not need to be a neural network, as long as it exposes a
differentiable function $p_{\theta_r}(z|f(X;\theta_f)=s)$ of sufficient capacity
to represent the true distribution. Fig.~\ref{fig:architecture} illustrates
a concrete example where $p_{\theta_r}(z | f(X;\theta_f)=s)$ is a mixture of
gaussians, as modeled with a mixture density network~\citep{bishop1994mixture}).
The $j^\textrm{th}$ output corresponds to the estimated value of the
corresponding parameter $\gamma_j$ of that distribution (e.g., the mean,
variance and mixing coefficients of its components). The estimated probability density
$p_{\theta_r}(z|f(X;\theta_f)=s)$ can then be evaluated for any $z \in {\cal Z}$ and any score $s \in {\cal S}$.


As with generative adversarial networks, we propose to
train $f$ and $r$ simultaneously, which we carry out by considering
the value function
\begin{equation}
    E(\theta_f, \theta_r) = {\cal L}_f(\theta_f) - {\cal L}_r(\theta_f, \theta_r)
\end{equation}
that we optimize by finding the minimax solution
\begin{equation}\label{eqn:min_thetaf}
    \smash{\hat\theta_f, \hat\theta_r} = \arg \min_{\theta_f} \max_{\theta_r} E(\theta_f, \theta_r).
\end{equation}
Without loss of generality, the adversarial training procedure to obtain
$(\smash{\hat\theta_f}, \smash{\hat\theta_r})$ is formally presented in
Algorithm~\ref{alg:adversarial-training} in the case of a binary classifier $f :
\mathbb{R}^p \to [0,1]$ modeling $p(Y=1|X)$. For reasons further explained
in Sec.~\ref{sec:theory}, ${\cal L}_f$ and ${\cal L}_r$  are respectively set to the
expected value of the
negative log-likelihood of $Y|X$ under $f$ and of $Z|f(X;\theta_f)$ under
$r$:
\begin{align}
    {\cal L}_f(\theta_f) &= \mathbb{E}_{x \sim X}  \mathbb{E}_{y \sim Y|x} [ -\log p_{\theta_f} (y|x) ], \\
    {\cal L}_r(\theta_f, \theta_r) &= \mathbb{E}_{s \sim f(X;\theta_f)}  \mathbb{E}_{z \sim Z|s} [-\log p_{\theta_r} (z|s)].
\end{align}
The optimization algorithm consists in using stochastic gradient descent
alternatively for solving Eqn.~\ref{eqn:min_thetaf}.
Finally, in the case of a class conditional pivot, the settings are the
same, except that the adversarial term ${\cal L}_r(\theta_f, \theta_r)$ is restricted to $Y=y$.


\section{Theoretical results}
\label{sec:theory}

In this section, we show that in the setting of
Algorithm~\ref{alg:adversarial-training} where ${\cal L}_f$ and ${\cal L}_r$ are
respectively set to expected value of the negative log-likelihood of $Y|X$ under
$f$ and of $Z|f(X;\theta_f)$ under $r$, the minimax solution of Eqn.~\ref{eqn:min_thetaf}
corresponds to a classifier
$f$ which is a pivotal quantity.

In this setting, the nuisance parameter $Z$ is considered as a random variable
of prior $p(Z)$, and our goal is to find a function
$f(\cdot;\theta_f)$ such that $f(X;\theta_f)$ and $Z$ are independent random
variables.   Importantly, classification of $Y$ with respect to $X$ is
considered in the context where $Z$ is marginalized out, which means that the
classifier minimizing ${\cal L}_f$ is optimal with respect to $Y|X$, but not
necessarily with $Y|X,Z$.
Results hold for a nuisance parameters $Z$ taking either
categorical or continuous values. By abuse of notation, $H(Z)$ denotes the
differential entropy in this latter case. Finally, the  proposition below is
derived in a non-parametric setting, by assuming that both $f$ and $r$ have
enough capacity.

\begin{proposition}\label{prop:2}
If there exists a minimax solution $(\smash{\hat\theta}_f, \smash{\hat\theta}_r)$
for Eqn.~\ref{eqn:min_thetaf} such that
$E(\hat\theta_f, \hat\theta_r) = H({Y|X}) - H(Z)$, then
$f(\cdot;\smash{\hat\theta}_f)$ is both an optimal classifier and a pivotal
quantity.
\end{proposition}

\begin{proof}

For fixed $\theta_f$, the adversary $r$ is optimal at
\begin{equation}
    \hat{\hat\theta}_r = \arg \max_{\theta_r} E(\theta_f, \theta_r)  = \arg \min_{\theta_r} {\cal L}_r(\theta_f, \theta_r),
\end{equation}
in which case $p_{\hat{\hat\theta}_r}(z|f(X;\theta_f)=s) =
p(z|f(X;\theta_f)=s)$ for all $z$ and all $s$, and ${\cal L}_r$ reduces to the expected entropy
$\mathbb{E}_{s \sim f(X;\theta_f)} [ H({Z|f(X;\theta_f)=s}) ]$ of the conditional distribution of the nuisance parameters.
This expectation corresponds to the conditional entropy of the random variables
$Z$ and $f(X;\theta_f)$ and can be written as $H(Z|f(X;\theta_f))$.
Accordingly, the
value function $E$ can be restated as a function depending on $\theta_f$ only:
\begin{equation}
    E'(\theta_f) = {\cal L}_f(\theta_f) -  H({Z|f(X;\theta_f)}).
\end{equation}
In particular, we have the lower bound
\begin{equation}
    H({Y|X}) - H(Z) \leq {\cal L}_f(\theta_f) - H({Z|f(X;\theta_f)})
\end{equation}
where the equality holds at $\smash{\hat\theta}_f = \arg \min_{\theta_f}
E'(\theta_f)$  when:
\begin{itemize}
    \item $\smash{\hat\theta}_f$ minimizes the negative log-likelihood of $Y|X$ under $f$,
    which happens when $\smash{\hat\theta}_f$ are the parameters
    of an optimal classifier. In this case, ${\cal L}_f$ reduces to its
    minimum value $H({Y|X})$.

    \item $\smash{\hat\theta}_f$ maximizes the conditional entropy
    $H({Z|f(X;\theta_f)})$, since $H(Z|f(X;\theta)) \leq H(Z)$ from the properties of entropy. Note that this
    latter inequality holds for both the discrete and the differential definitions of entropy.
\end{itemize}
By assumption, the lower bound is active, thus we have $H(Z|f(X;\theta_f)) = H(Z)$
because of the second condition, which happens exactly when $Z$ and $f(X;\theta_f)$
are independent variables. In other words,  the
optimal classifier $f(\cdot;\smash{\hat\theta}_f)$ is also a pivotal
quantity.
\end{proof}

Proposition~\ref{prop:2} suggests that if at each step of
Algorithm~\ref{alg:adversarial-training} the adversary $r$ is allowed to reach
its optimum given $f$ (e.g., by setting $K$ sufficiently high) and if $f$ is
updated to improve ${\cal L}_f(\theta_f) -  H({Z|f(X;\theta_f)})$ with
sufficiently small steps, then $f$ should converge to a classifier that is both
optimal and pivotal, provided such a classifier exists. Therefore, the adversarial term ${\cal L}_r$
can be regarded as a way to select among
the class of all optimal classifiers a function $f$ that is also pivotal.
Despite the former theoretical characterization of the minimax solution
of Eqn.~\ref{eqn:min_thetaf}, let us note that formal guarantees of
convergence towards that solution by
Algorithm~\ref{alg:adversarial-training} in the case where a finite number $K$
of steps is taken for $r$ remains to be proven.

In practice, the assumption of existence of an optimal and pivotal classifier may
not hold because the nuisance parameter directly shapes the decision boundary.
In this case, the lower bound
\begin{equation}\label{eqn:lower-bound-strict}
    H({Y|X}) - H(Z) < {\cal L}_f(\theta_f) - H({Z|f(X;\theta_f)})
\end{equation} is strict: $f$ can either be an optimal classifier or a
pivotal quantity, but not both simultaneously. In this situation, it is natural
to rewrite the value function $E$  as
\begin{equation}\label{eqn:vf-lambda}
    E_\lambda(\theta_f, \theta_r) = {\cal L}_f(\theta_f) - \lambda {\cal L}_r(\theta_f, \theta_r),
\end{equation}
where $\lambda \geq 0$ is a hyper-parameter controlling the trade-off between
the performance of $f$ and its independence with respect to the nuisance
parameter. Setting $\lambda$ to a large value will preferably enforces $f$ to be
pivotal while setting $\lambda$ close to $0$ will rather constraint $f$ to be
optimal. When the lower bound is strict, let us note however that there may
exist distinct but equally good solutions $\theta_f,\theta_r$ minimizing
Eqn.~\ref{eqn:vf-lambda}. In this zero-sum game, an increase in accuracy would
exactly be compensated by a decrease in pivotality and vice-versa. How to best
navigate this Pareto frontier to maximize a higher-level objective remains a
question open for future works.

Interestingly, let us finally emphasize that our results hold using only the
(1D) output $s$ of $f(\cdot;\theta_f)$ as
input to the adversary. We could similarly enforce an intermediate
representation of the data to be pivotal, e.g. as in
\citep{ganin2014unsupervised}, but this is not necessary.


\section{Experiments}

In this section, we empirically demonstrate the effectiveness of the approach
with a toy example and examples from particle physics. Notably,
there are no other other approaches to compare to in the case of continuous
nuisance parameters, as further explained in Sec.~\ref{sec:related}. In the case
of binary parameters, we do not expect results to be much different from
previous works.

\subsection{A toy example with a continous nuisance parameter}
\label{sec:toy}

\begin{figure*}
    \begin{center}
        \includegraphics[width=0.245\textwidth]{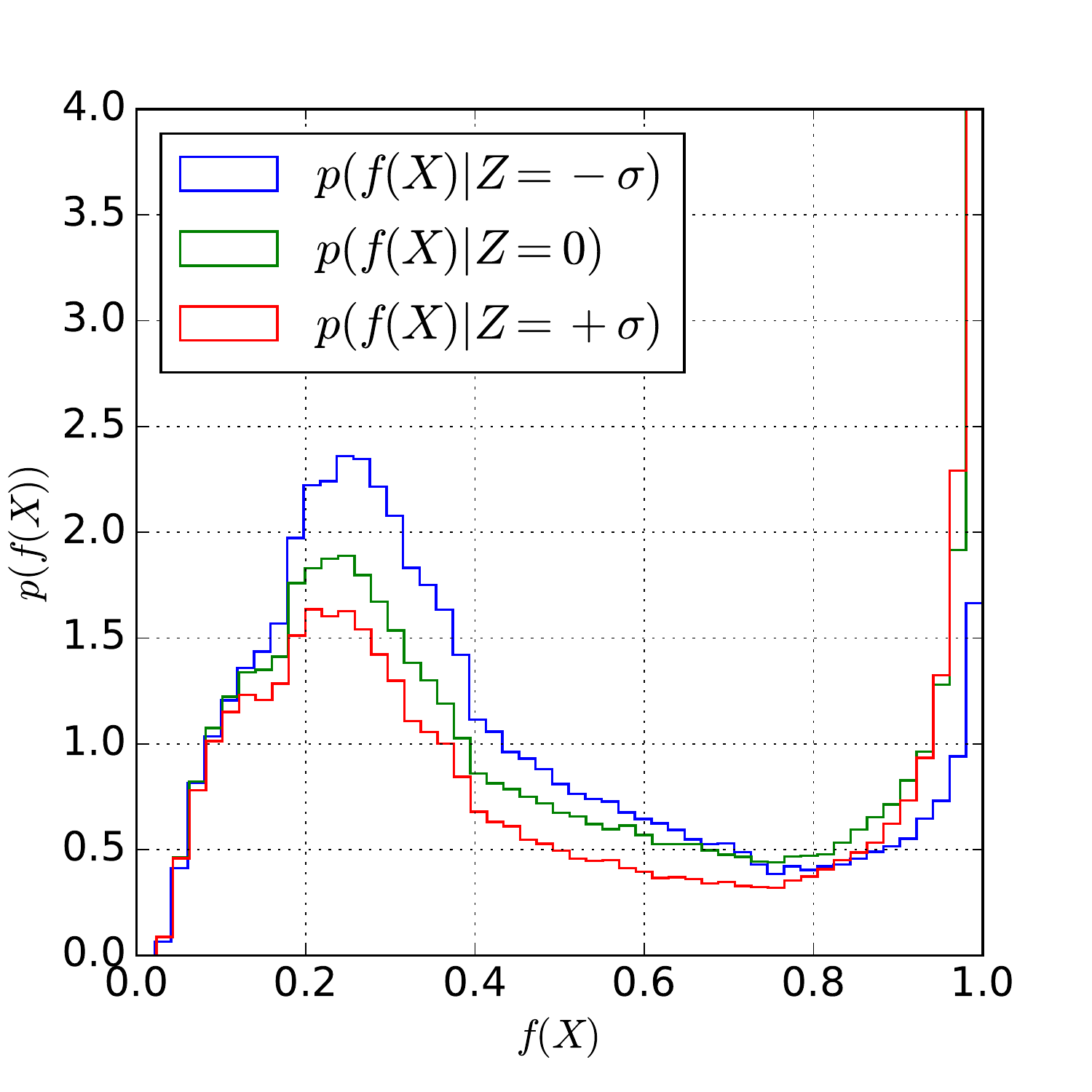}
        \includegraphics[width=0.245\textwidth]{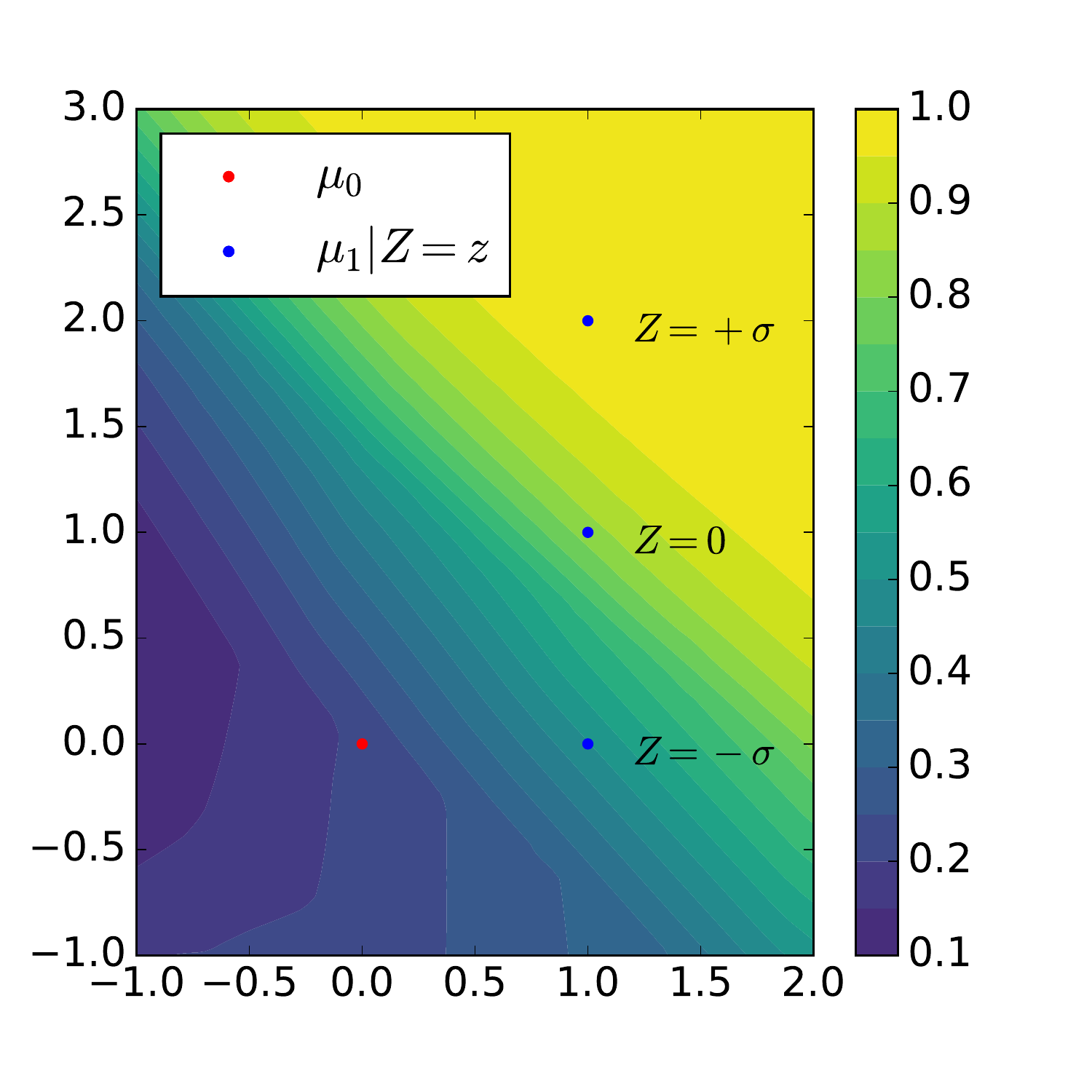}
        \includegraphics[width=0.245\textwidth]{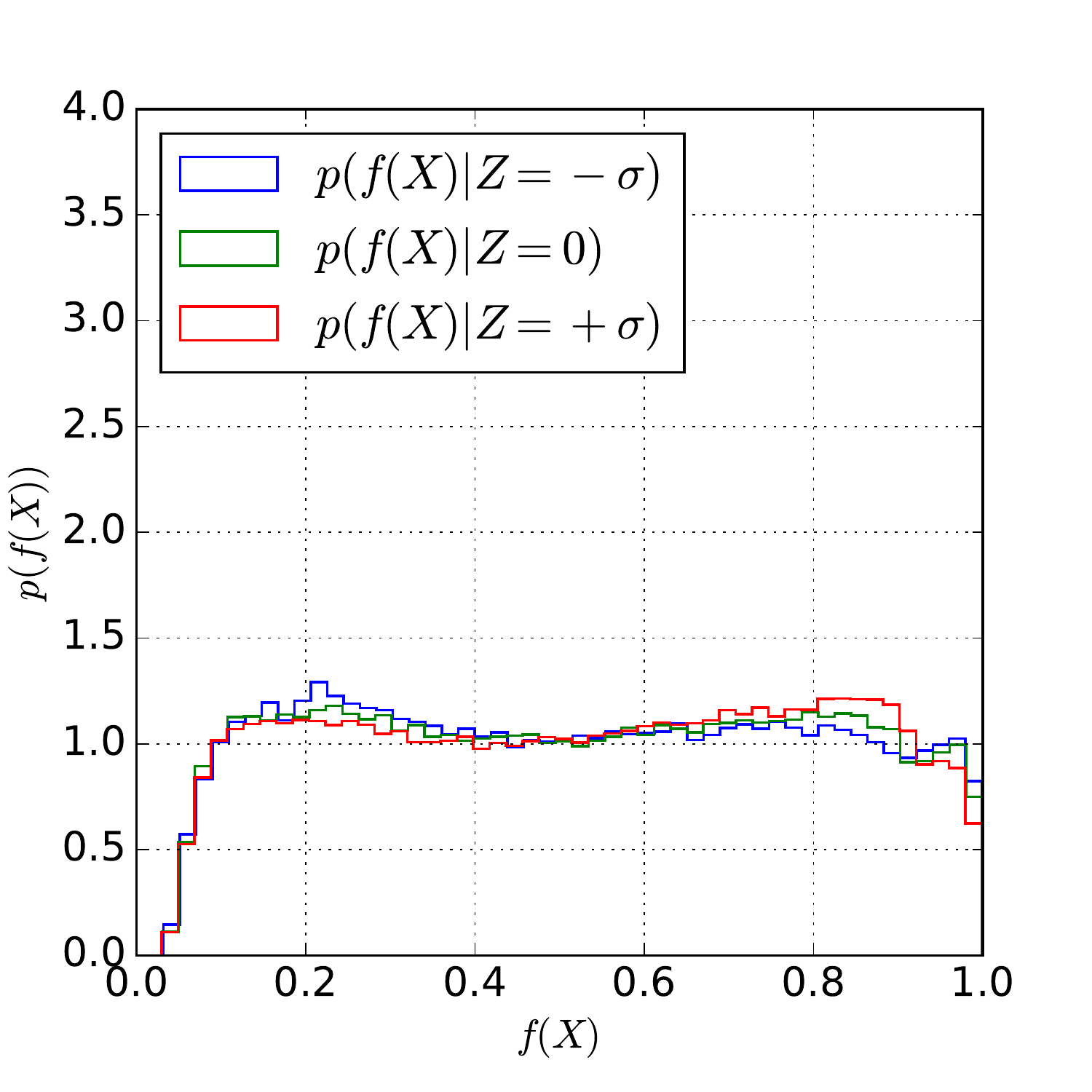}
        \includegraphics[width=0.245\textwidth]{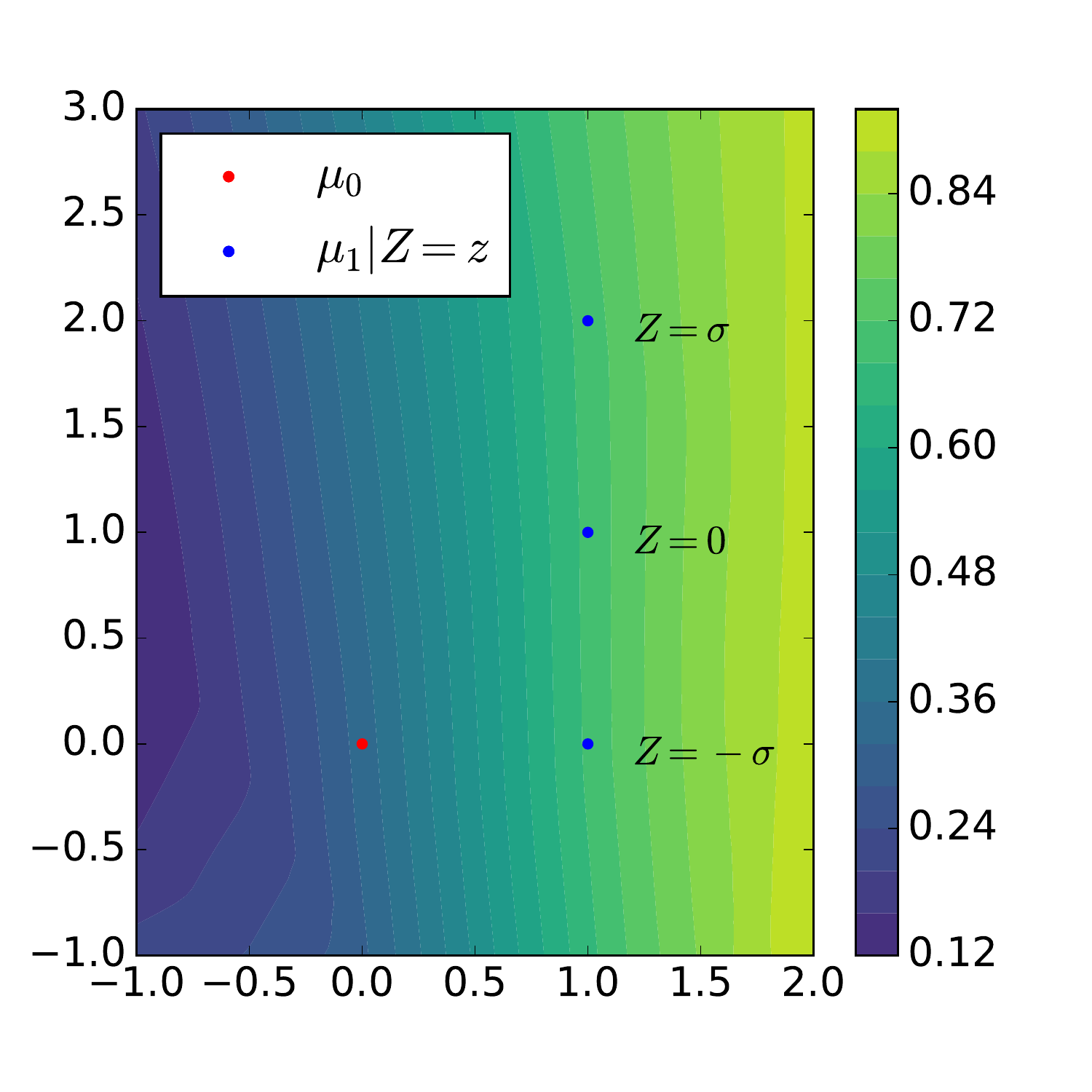}
    \end{center}

    \caption{Toy example.
    (Left) Conditional probability densities of the decision scores at $Z=-\sigma, 0, \sigma$
        without adversarial training. The resulting densities
       are dependent on the continuous parameter $Z$, indicating that $f$ is not pivotal.
    (Middle left) The associated decision surface, highlighting
       the fact that samples are easier to classify for values of $Z$ above $\sigma$,
       hence explaining the dependency.
    (Middle right) Conditional probability densities of the decision scores at $Z=-\sigma, 0, \sigma$ when $f$ is
       built with adversarial training.
       The resulting densities are now almost identical to each other, indicating only a
       small dependency on $Z$.
    (Right) The associated decision surface, illustrating how adversarial
       training bends the decision function vertically to erase the dependency on $Z$.
    }
    \label{fig:toy}
\end{figure*}

As a guiding toy example, let us consider the binary classification of 2D
data drawn from multivariate gaussians with equal priors, such that
\begin{align}
    x &\sim {\cal N}\left ((0,0), \begin{bmatrix}
                              1 & -0.5 \\
                              -0.5 & 1
                            \end{bmatrix}\right) &\text{ when } Y=0, \\
    x|Z=z &\sim {\cal N}\left ((1,1+z),  \begin{bmatrix}
                              1 & 0 \\
                              0 & 1
                             \end{bmatrix}\right) &\text{ when } Y=1.
\end{align}
The continuous nuisance parameter $Z$ here represents our
uncertainty about the location of the mean of the second gaussian. Our goal is to
build a classifier $f(\cdot;\theta_f)$ for predicting $Y$ given $X$, but such that
the probability distribution of $f(X;\theta_f)$ is invariant with respect to the
nuisance parameter $Z$.

Assuming a gaussian prior $z \sim {\cal N}(0,1)$, we generate
data $\{ x_i, y_i, z_i \}_{i=1}^N$, from which we train a neural
network  $f$ minimizing ${\cal L}_f(\theta_f)$ without considering its
adversary $r$. The network architecture comprises 2 dense hidden layers of 20
nodes respectively with tanh and ReLU activations, followed  by a dense output layer with a single
node with a sigmoid activation. As shown in Fig.~\ref{fig:toy}, the resulting
classifier is not pivotal, as the conditional probability densities of its
decision scores $f(X;\theta_f)$ show large discrepancies between values $z$ of
the nuisance parameters. While not shown here, a classifier trained only from data
generated at the nominal value $Z=0$ would also not be pivotal.

Let us now consider the joint training of $f$ against an adversary $r$
implemented as a mixture density network modeling $Z|f(X;\theta_f)$ as a mixture
of five gaussians. The network architecture of $r$ comprises 2
dense hidden layers of 20 nodes with ReLU activations, followed by an
output layer of 15 nodes corresponding to the means, standard deviations and
mixture coefficients of the gaussians. Output nodes for the mean values
come with linear activations, output nodes for the standard deviations with
exponential activations to ensure positivity, while output nodes for the mixture
coefficients implement the softmax function to ensure positivity and
normalization. When running Algorithm~\ref{alg:adversarial-training} as
initialized with the classifier $f$ obtained previously, adversarial training
effectively reshapes the decision function so it that becomes almost independent
on the nuisance parameter, as shown in Fig.~\ref{fig:toy}. The
conditional probability densities of the decision scores $f(X;\theta_f)$ are now
very similar to each other, indicating only a residual  dependency on the
nuisance, as theoretically expected. The dynamics of adversarial training is
illustrated in Fig.~\ref{fig:toy-training}, where the losses ${\cal L}_f$,
${\cal L}_r$ and ${\cal L}_f - \lambda {\cal L}_r$ are evaluated after each
iteration. In the first iterations,
we observe that the global objective ${\cal L}_f - \lambda {\cal L}_r$ is
minimized by making the classifier less accurate, hence the corresponding
increase of ${\cal L}_f$, but which results in a classifier that is more
pivotal, hence the associated increase of ${\cal L}_r$ and the total net
benefit. As learning goes, minimizing $E$ requires making predictions that
are more accurate, hence decreasing ${\cal L}_f$, or that are even less
dependent on $Z$, hence shaping $p_{\theta_r}$ towards the prior $p(Z)$. Indeed,
${\cal L}_f$ eventually starts decreasing, while remaining bounded from below by
$\min_{\theta_f} {\cal L}_f(\theta_f)$ as approximated by the dashed line in the
first plot. Similarly,  ${\cal L}_r$ tends towards the differential entropy
$H(Z)$ of the prior (where $H(Z) = \log(\sigma \sqrt{2 \pi e}) = 1.419$ in the
case of a standard normal), as shown by the dashed line in the
second plot. Finally, let us note that the ideal situation of a classifier that
is both optimal and pivotal is unreachable for this problem, as shown
in the third plot by the  offset between ${\cal L}_f - \lambda {\cal L}_r$ and
the dashed line approximating $H({Y|X}) - \lambda H(Z)$.

\begin{figure}
\centering
\begin{minipage}{.46\linewidth}
    \centering
        \includegraphics[width=\textwidth]{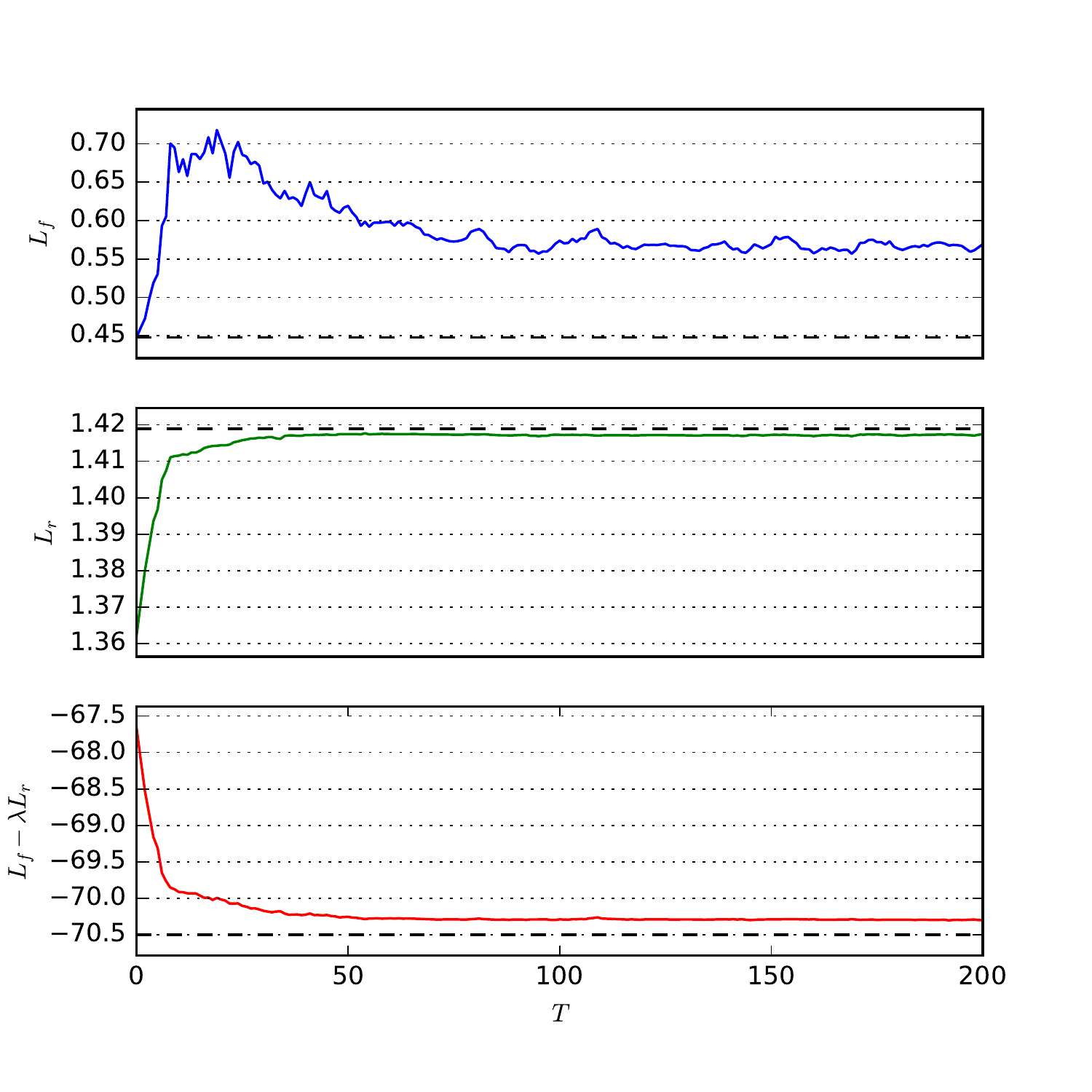}\vspace{-1em}
        \caption{Toy example. Training curves for ${\cal L}_f(\theta_f)$, ${\cal L}_r(\theta_f, \theta_r)$
                 and ${\cal L}_f(\theta_f) - \lambda {\cal L}_r(\theta_f, \theta_r)$.
                 Initialized with a pre-trained classifier $f$, adversarial training was performed for $200$ iterations, mini-batches of size $M=128$, $K=500$ and $\lambda=50$.}
        \label{fig:toy-training}
\end{minipage}
\hspace{.05\linewidth}
\begin{minipage}{.46\linewidth}
    \centering
        \includegraphics[width=\textwidth]{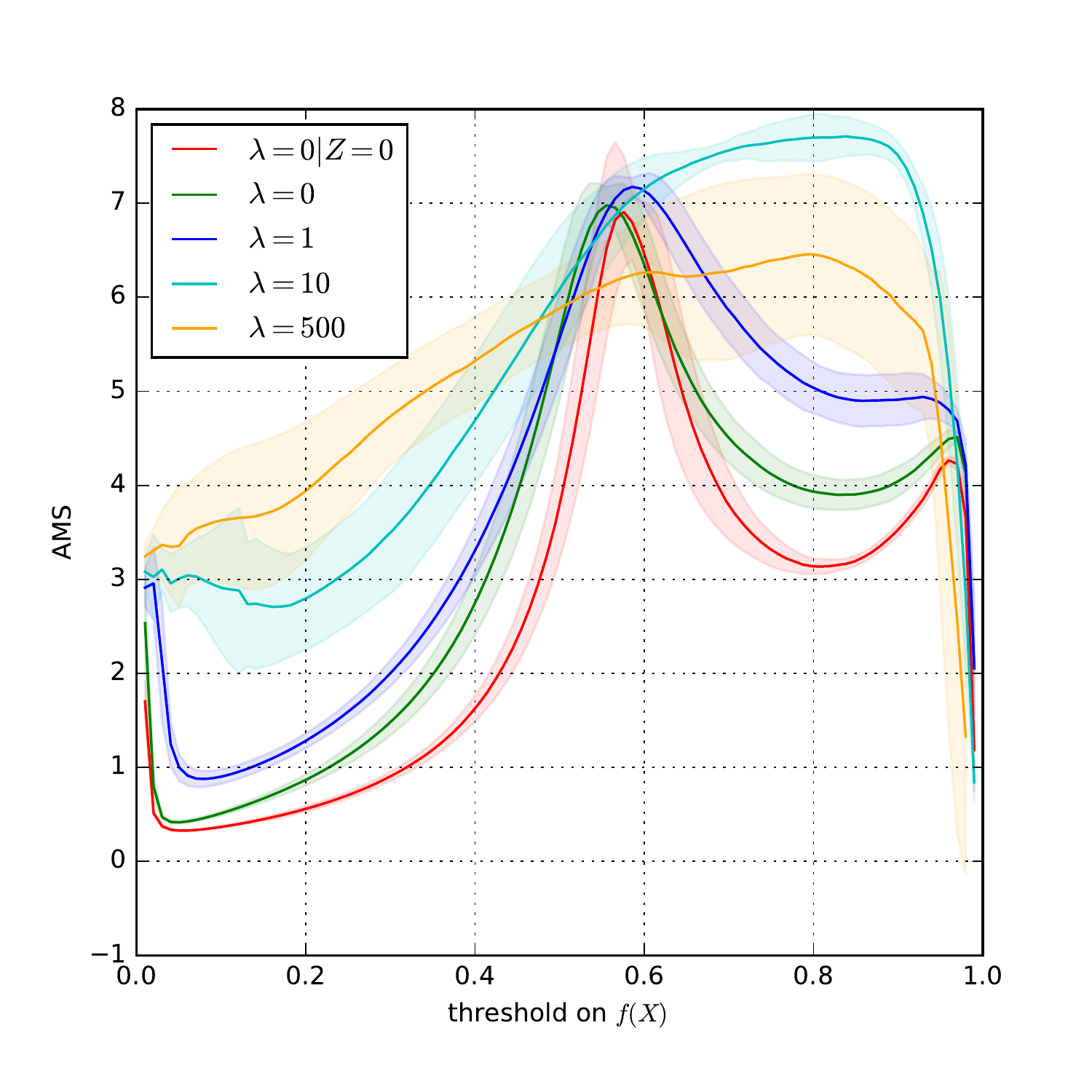}\vspace{-1em}
        \caption{Physics example. Approximate median significance as a function of the decision threshold
                 on the output of $f$. At $\lambda=10$, trading
                 accuracy for independence to pileup
                 results in a net benefit in terms of statistical significance.}
        \label{fig:physics-ams}
\end{minipage}
\end{figure}

\subsection{High energy physics examples}
\label{sec:hep}

\paragraph{Binary Case} Experiments at high energy colliders like the LHC \citep{LHCMachine} are
searching for evidence of new particles beyond those described by the
Standard Model (SM) of particle physics. A wide array of
theories predict the existence of new massive particles that would decay to
known particles in the SM such as the $W$ boson. The $W$ boson is unstable and
can decay to two quarks, each of which produce collimated sprays of particles
known as jets.
If the exotic particle
is heavy, then the $W$ boson will be moving very fast, and  relativistic effects
will cause the two jets from its decay to merge into a single `$W$-jet'. These
$W$-jets have a rich internal substructure.
However,
jets are also produced ubiquitously at high energy colliders through more
mundane processes in the SM, which leads to a challenging classification problem
that is beset with a number of sources of systematic uncertainty.
The classification challenge used here is common in jet substructure
studies (see e.g. \citep{Khachatryan:2014vla,ATL-PHYS-PUB-2015-033,wbosonATLAS}): we aim to distinguish normal jets produced copiously
at the LHC ($Y=0$) and from $W$-jets ($Y=1$) potentially coming from an exotic process.
We reuse the datasets used in
\citep{baldi2016jet}.

Challenging in its own right, this classification
problem is made all the more difficult by the presence of pileup, or multiple
 proton-proton interactions occurring simultaneously with the primary
interaction.  These pileup interactions produce additional particles that can
contribute significant energies to jets unrelated to the underlying
discriminating information. The number of pileup interactions can vary with
the running conditions of the collider, and we want the classifier to be robust
to these conditions. Taking some liberty, we consider an extreme case with
a categorical nuisance parameter, where $Z=0$ corresponds to events without pileup
and $Z=1$ corresponds to events with pileup, for which there are an average  of
50 independent pileup interactions overlaid.

We do not expect that we will be able to find a function $f$ that simultaneously
minimizes the classification loss ${\cal L}_f$ and is pivotal. Thus, we need to optimize
the hyper-parameter $\lambda$ of Eqn.~\ref{eqn:vf-lambda} with respect to
a higher-level objective. In this case, the natural higher-level context is a
hypothesis test of a null hypothesis with no $Y=1$ events against an
alternate hypothesis that is a mixture of $Y=0$ and $Y=1$ events.
In the absence of systematic uncertainties, optimizing ${\cal L}_f$ simultaneously
optimizes the power of a classical hypothesis test in the Neyman-Pearson sense.
When we include systematic uncertainties we need to balance the
classification performance against the robustness to uncertainty in $Z$.
Since we are still performing a hypothesis test against the null, we only
wish to impose the pivotal property on $Y=0$ events. To this end,
we use as a higher level objective the Approximate Median Significance (AMS), which is a natural generalization
of the power of a hypothesis test when systematic uncertainties are taken into account
(see Eqn.~20 of \cite{adam2014higgs}).

For several values of $\lambda$, we train a classifier
using  Algorithm~\ref{alg:adversarial-training} but consider the adversarial
term ${\cal L}_r$ conditioned on $Y=0$ only, as outlined in
Sec.~\ref{sec:problem}. The architecture of $f$ comprises 3 hidden layers of 64 nodes
respectively with tanh, ReLU and ReLU activations, and is terminated by a single final
output node with a sigmoid activation. The architecture of $r$ is the same,
but uses only ReLU activations in its hidden nodes.
As in the previous example, adversarial training is initialized
with $f$ pre-trained. Experiments are performed on a subset of
150000 samples for training while AMS is evaluated on an independent test set of
5000000 samples. Both training and testing samples are weighted such that the
null hypothesis corresponded to 1000  of $Y=0$ events and the alternate
hypothesis included an additional 100 $Y=1$ events prior to any thresholding on $f$.
This allows us to probe the efficacy of the
method proposed here in a representative background-dominated high energy physics environment.
Results reported below are averages over 5 runs.

As  Fig.~\ref{fig:physics-ams} illustrates, without adversarial training (at
$\lambda=0|Z=0$ when building a classifier at the nominal value $Z=0$ only, or
at $\lambda=0$ when building a classifier on data sampled from $p(X,Y,Z)$), the
AMS peaks at $7$. By contrast, as the pivotal constraint
is made stronger (for $\lambda > 0$) the AMS peak moves higher, with a maximum
value around  $7.8$ for $\lambda=10$. Trading classification
accuracy for robustness to pileup thereby results in a net
benefit in terms of the power of the hypothesis test. Setting $\lambda$ too high however
(e.g. $\lambda=500$) results in a decrease of the maximum AMS, by
focusing the capacity of $f$ too strongly on independence with $Z$, at the expense of
accuracy. In effect, optimizing $\lambda$
yields a principled and effective approach to control the trade-off between
accuracy and robustness that ultimately maximizes the
power of the enveloping hypothesis test.

\paragraph{Continous Case}
Recently, an independent group has used our approach to learn jet classifiers
that are independent of the jet mass~\citep{Shimmin:2017mfk}, which is a continuous attribute.
The results of their studies show that the adversarial training strategy works very well for
real-world problems with continuous attributes, thus enhancing the sensitivity of searches for
new physics at the LHC.


\section{Related work}
\label{sec:related}

Learning to pivot can be related to the problem
of domain adaptation
\citep{blitzer2006domain,pan2011domain,gopalan2011domain,gong2013connecting,baktashmotlagh2013unsupervised,ajakan2014domain,ganin2014unsupervised},
where the goal is often stated as trying to learn a domain-invariant
representation of the data. Likewise, our method also relates to the problem of
enforcing fairness in classification
\citep{kamishima2012fairness,zemel2013learning,feldman2015certifying,EdwardsS15,zafar2015fairness,louizos2015variational},
which is stated as learning a classifier that is independent of some chosen
attribute such as gender, color or age. For both families of methods, the
problem can equivalently be stated as learning a classifier which is a pivotal
quantity with respect to either the domain or the selected feature. As an
example, unsupervised domain adaptation with labeled data from a source domain
and unlabeled data from a target domain can be recast as learning a predictive
model $f$ (i.e., trained to minimize ${\cal L}_f$ evaluated on labeled source
data only) that is also a pivot with respect to the domain $Z$ (i.e., trained to
maximize ${\cal L}_r$ evaluated on both source and target data). In this
context, \citep{ganin2014unsupervised,EdwardsS15} are certainly among the
closest to our work, in which domain invariance and fairness are enforced
through an adversarial minimax setup composed of a classifier and an adversarial
discriminator. Following this line of work, our method can be regarded as a unified generalization that also supports
a continuously parametrized family of domains or as enforcing fairness over
continuous attributes.

Most related
work is based on the strong and limiting assumption that $Z$ is a binary
random variable (e.g., $Z=0$ for the source domain, and
$Z=1$ for the target domain). In particular, \citep{pan2011domain,gong2013connecting,baktashmotlagh2013unsupervised,zemel2013learning,ganin2014unsupervised,ajakan2014domain,EdwardsS15,louizos2015variational}
are all based on the minimization of some form of divergence between the two distributions of
$f(X)|Z=0$ and $f(X)|Z=1$. For this reason,
these works cannot directly be generalized to non-binary or
continuous nuisance parameters, both from a practical and theoretical point of view.
Notably, \cite{kamishima2012fairness} enforces fairness through
a prejudice regularization term based on empirical estimates of $p(f(X)|Z)$.
While this approach is in principle sufficient for handling non-binary
nuisance parameters $Z$, it requires accurate
empirical estimates of $p(f(X)|Z=z)$ for all values $z$, which quickly becomes
impractical as the cardinality of $Z$ increases. By contrast, our approach models the conditional
dependence through an adversarial network, which allows for generalization without
necessarily requiring a growing number of training examples.

A common approach to account for systematic uncertainties in
a scientific context (e.g. in high energy physics)
is to take as fixed a classifier $f$ built from training data for a nominal
value $z_0$ of the nuisance parameter, and then propagate uncertainty by
estimating $p(f(x)|z)$ with a parametrized calibration procedure. Clearly, this
classifier is however not optimal for $z \neq z_0$. To overcome this issue, the
classifier $f$ is sometimes built instead on a mixture of training data
generated from several plausible values $z_0, z_1, \dots$ of the nuisance
parameter. While this certainly improves classification performance with respect
to the marginal model $p(X,Y)$, there is no reason to expect the resulting
classifier to be pivotal, as shown previously in Sec.~\ref{sec:toy}. As an
alternative, parametrized
classifiers~\citep{cranmer2015approximating,Baldi:2016fzo} directly take
(nuisance) parameters as additional input variables, hence ultimately providing
the most statistically powerful approach for incorporating the effect of
systematics on the underlying classification task.
In practice, parametrized
classifiers  are also computationally expensive to build and evaluate. In
particular, calibrating their decision function, i.e. approximating
$p(f(x,z)|y,z)$ as a continuous function of $z$, remains an open challenge. By
contrast, constraining $f$ to be pivotal yields a classifier
that can
be directly used in a wider range of applications, since the
dependence on the nuisance parameter $Z$ has already been eliminated.


\section{Conclusions}
\label{sec:conclusions}

In this work, we proposed a flexible learning procedure for building a
predictive model that is independent of continuous or categorical nuisance
parameters by jointly training two neural networks in an adversarial fashion.
From a theoretical perspective, we motivated the proposed algorithm by showing
that the minimax value  of its value function corresponds to a predictive model
that is both optimal and pivotal (if that models exists) or for which one can
tune the trade-off between power and robustness. From an empirical point of
view, we confirmed the effectiveness of our method on a toy example
and a particle physics example.


In terms of applications, our solution can be used in any situation
where the training data may not be representative of the real data the
predictive model will be applied to in practice. In the scientific context, the
presence of systematic uncertainty can be incorporated by considering a family
of data generation processes, and it would be worth revisiting those scientific problems
that utilize machine learning in light of this technique. Moreover, the approach
also extends to cases where independence of the predictive model with respect to
observed random variables is desired, as in fairness for
classification.


\bibliography{bibliography}
\bibliographystyle{apalike}

\end{document}